\documentclass[11pt,a4paper, svgnames]{article}
\usepackage[hyperindex,breaklinks]{hyperref}

\usepackage[hyperref]{emnlp2020}
\usepackage{times}
\usepackage{latexsym}

\usepackage{longtable}
\usepackage{comment}
\usepackage{nicefrac}
\usepackage{amsmath}
\usepackage{amssymb}
\usepackage{amsfonts}
\usepackage{amsthm}
\usepackage{mathtools}
\usepackage{booktabs}
\usepackage{tabularx}
\usepackage{multirow}
\usepackage{relsize}
\usepackage{longtable}
\usepackage{xtab}
\usepackage{xspace}
\usepackage{enumitem}
\usepackage{textcomp}
\usepackage{xcolor}
\usepackage{chemarrow}
\usepackage{graphicx}
\usepackage{stmaryrd}
\usepackage{pifont}
\usepackage{colortbl}
\usepackage{tabstackengine}
\usepackage{adjustbox}
\usepackage{tikz-dependency}
\usepackage{stackengine, scalerel}
\usepackage{calc}

\usepackage{microtype}

\aclfinalcopy % Uncomment this line for the final submission
\def\aclpaperid{1943} %  Enter the acl Paper ID here

\usepackage{tikz}
\usetikzlibrary{arrows.meta}
\usetikzlibrary{shapes.arrows}
\usetikzlibrary{shapes.misc}
\usetikzlibrary{fit}

\newcommand{\Cgreedy}{DarkGoldenrod}
\newcommand{\Ccyc}{Crimson}
\newcommand{\Cma}{Green}
\newcommand{\Cexit}{Fuchsia}
\newcommand{\Center}{MediumBlue}
\newcommand{\Cexternal}{Brown}
\definecolor{verydarkgray}{gray}{0.45}
\newcommand{\Cdead}{verydarkgray}
\definecolor{darkgray}{gray}{0.6}

\definecolor{commentcolor}{rgb}{0.4, 0.22, 0.33}

\usepackage{algorithm}
\usepackage{algorithmicx}
\usepackage[noend]{algpseudocode}
\algnewcommand{\parState}[1]{\State%
    \parbox[t]{\dimexpr\linewidth-\algmargin}{\strut\hangindent=\algorithmicindent \hangafter=1 #1\strut}}
\algrenewcommand\algorithmicindent{1.0em}%
\newcommand{\algorithmicdowhile}{\textbf{do}:}
\algdef{SE}[DOWHILE]{Do}{doWhile}{\algorithmicdowhile}[1]{\algorithmicwhile\ #1}%
\newcommand{\algorithmicfunc}[1]{\textbf{def} #1 :}
\algdef{SE}[FUNC]{Func}{EndFunc}[1]{\algorithmicfunc{#1}}{}

\newif\ifboldnumber

\algrenewcommand\alglinenumber[1]{%
  \footnotesize\ifboldnumber\color{red}\bfseries\fi\global\boldnumberfalse#1:}
\MakeRobust{\Call}   % fix algpseudocode bug with nested \Call commands

\newcommand{\rightcomment}[1]{{\color{commentcolor} \(\triangleright\) {\footnotesize\textit{#1}}}}
\algrenewcommand{\algorithmiccomment}[1]{\hfill \rightcomment{#1}}  % redefines \Comment
\algnewcommand{\LineComment}[1]{\State \rightcomment{#1}}
\algnewcommand{\LinesComment}[1]{\State \rightcomment{\parbox[t]{\linewidth-\leftmargin-\widthof{\(\triangleright\) }}{#1}}}

\renewcommand\algorithmicthen{:}

\makeatletter
\ifthenelse{\equal{\ALG@noend}{t}}%
  {\algtext*{EndFunc}}
  {}%
\makeatother

\algnewcommand{\IIf}[1]{\State\algorithmicif\ #1\ \algorithmicthen}
\algnewcommand{\EndIIf}{\unskip}

\newcommand{\constrain}[2]{\mathrm{constrain}_{#2}(#1)}
\newcommand{\stitch}[2]{#1 \looparrowright #2}

\newcommand{\cg}{{C}}

\newcommand{\cc}{{c}}

\newcommand{\critical}{{critical cycle}\xspace}
\newcommand{\criticals}{{critical cycles}\xspace}

\newcommand{\prov}[0]{\pi}

\newcommand{\arbsk}[3]{{\mathcal{A}_{#1}^{#2}(#3)}}
\newcommand{\arbsone}[2]{\arbsk{#1}{\dagger}{#2}}

\newcommand{\Wenter}{{\color{\Center} \textbf{enter}}\xspace}
\newcommand{\Wexit}{{\color{\Cexit} \textbf{exit}}\xspace}
\newcommand{\Wexternal}{{\color{\Cexternal} \textbf{external}}\xspace}
\newcommand{\Wdead}{{\color{\Cdead} \textbf{dead}}\xspace}
\newcommand{\entersite}{{entrance site}\xspace}

\newcommand{\dottedarrow}[0]{\tikz[>=latex,baseline=-0.5ex]{ \path[->] (0,0) edge[dotted, thick] (0.6,0);}\xspace}

\newcommand{\dashedarrow}[0]{\tikz[>=latex,baseline=-0.5ex, \Cdead]{ \path[->] (0,0) edge[dashed, thick] (0.6,0);}\xspace}

\newcommand{\rarrow}[1]{\xrightarrow{\raisebox{-0.5ex}[0ex][0ex]{\tiny $#1$}}}

\newcommand{\edge}[3]{(#1\!\rarrow{{#2}}\!#3)}
\newcommand{\tree}{{A}}
\newcommand{\treep}{{\tree'}}

\newcommand{\arbs}[2]{{\mathcal{A}_{#1}(#2)}}

\renewcommand{\root}{\rho}
\newcommand{\bigo}[1]{\mathcal{O}(#1)}

\newcommand{\cle}{CLE\xspace}

\newcommand{\abs}[1]{\lvert #1 \rvert}
\newcommand{\treecost}[1]{\bar{w}(#1)}

\newcommand{\greedy}[2]{\overrightarrow{#1^{#2}}}
\newcommand{\mst}[2]{{#1^{*}_{#2}}}
\newcommand{\opt}[2]{\mathrm{opt}_{#2}\!\left({#1}\right)}
\newcommand{\contract}[2]{{#1}_{\!/{#2}}}
\newcommand{\bestDep}[1]{{{#1}^\dagger}}

\makeatletter
\newcommand{\oset}[3][0ex]{%
  \mathrel{\mathop{#3}\limits^{
    \vbox to#1{\kern-2\ex@
    \hbox{$\scriptscriptstyle#2$}\vss}}}}
\makeatother

\newcommand{\msetminus}[0]{{\backslash\!\!\backslash}}

\newcommand{\multidelete}[2]{#1 \msetminus #2}

\newcommand{\ar}{arborescence\xspace}
\newcommand{\ars}{arborescences\xspace}

\newcommand{\Ars}{Arborescences\xspace}

\newcommand{\defeq}[0]{\overset{\smaller\mathrm{def}}{=}}
\newcommand{\defn}[1]{\textbf{#1}}

\renewcommand{\setminus}{\smallsetminus}
\renewcommand{\bar}[1]{\overline{#1}}

\DeclareMathOperator*{\argmax}{argmax\,}

\newcommand*\nodeid[1]{\tikz[baseline]{
            \node[shape=circle,draw,inner sep=1pt, text centered, text depth=1mm] () at (0,0.05) {\tiny $#1$};}}
\newcommand*\nodeId[1]{\tikz[baseline]{
            \node[shape=circle,draw,inner sep=1pt, text centered, text depth=0.2mm] () at (0,0.07) {\tiny $#1$};}}
\newcommand*\noderoot{\tikz[baseline]{
            \node[shape=circle,draw,inner sep=0pt, text centered, text depth=2mm, minimum size=0pt] () at (0,0.1) {\tiny $\root$};}}
          
\newcommand{\figrefmarker}{\workedexample}

\newcommand{\figref}{$^{\figrefmarker}$\xspace}

\usepackage{emoji}

\let\rightarrow\chemarrow % Change arrow head

\newtheorem{thm}{Theorem}
\newtheorem{cor}{Corollary}
\newtheorem{lemma}{Lemma}
\newtheorem{prop}{Proposition}

\theoremstyle{definition}

\newtheorem{innercustomthm}{Theorem}
\newenvironment{customthm}[1]
  {\renewcommand\theinnercustomthm{#1}\innercustomthm}
  {\endinnercustomthm}

\usepackage{cleveref}
\crefname{section}{\S}{\S\S}
\Crefname{section}{\S}{\S\S}
\crefname{table}{Tab.}{}
\crefname{figure}{Fig.}{}
\crefname{algorithm}{Alg}{}
\crefname{algorithm}{Alg}{}
\crefname{line}{Line}{}
\crefname{appendix}{App.}{}
\crefformat{section}{\S#2#1#3}
\crefname{thm}{Theorem}{}
\crefname{prop}{Proposition}{}
\crefname{defin}{Definition}{}
\crefname{lemma}{Lemma}{}
\crefname{cor}{Corollary}{}
\crefname{equation}{}{}

\makeatletter
\g@addto@macro\normalsize{%
  \setlength\abovedisplayskip{5pt}
  \setlength\belowdisplayskip{5pt}
  \setlength\abovedisplayshortskip{5pt}
  \setlength\belowdisplayshortskip{5pt}
}
\makeatother

\newcommand{\workedexample}{\emoji[twitter]{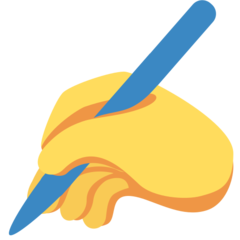}\xspace}

\title{Please Mind the Root: Decoding \Ars for Dependency Parsing}

\newcommand{\ucambridge}{\emoji[twitter]{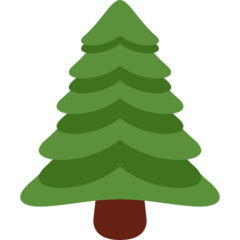}}
\newcommand{\ethz}{\emoji[twitter]{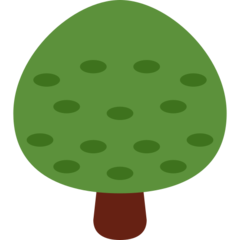}}
\newcommand{\jhu}{\emoji[twitter]{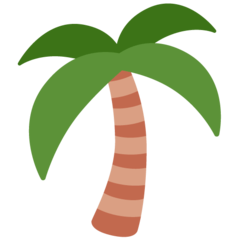}}

\author{
{Ran Zmigrod\raise1.0ex\hbox{\normalfont\ucambridge}\raise1.0ex\hbox{\normalfont}}~\;~Tim Vieira\raise1.0ex\hbox{\normalfont\jhu}~\;~Ryan Cotterell\raise1.0ex\hbox{\normalfont\ucambridge,\ethz}
\\
  \raise1.0ex\hbox{\normalfont\ucambridge}University of Cambridge~\;~\raise1.0ex\hbox{\normalfont\jhu}Johns Hopkins University~\;~\raise1.0ex\hbox{\normalfont\ethz}ETH Z\"{u}rich \\
  \texttt{rz279@cam.ac.uk}~\;~\texttt{tim.f.vieira@gmail.com} \\ \texttt{ryan.cotterell@inf.ethz.ch}
}

\date{}

\begin{document}
\maketitle
\begin{abstract}
The connection between dependency trees and spanning trees is exploited by the NLP community to train and to decode graph-based dependency parsers.
However, the NLP literature has missed an important difference between the two structures: only \emph{one} edge may emanate from the root in a dependency tree.
We analyzed the output of state-of-the-art parsers on many languages from the Universal Dependency Treebank: although these parsers are often able to learn that trees which violate the constraint should be assigned lower probabilities, their ability to do so unsurprisingly degrades as the size of the training set decreases.
In fact, the worst constraint-violation rate we observe is $24\%$.
Prior work has proposed an inefficient algorithm to enforce the constraint, which adds a factor of $n$ to the decoding runtime.  
We adapt an algorithm due to \citet{GabowT84} to dependency parsing, which satisfies the constraint without compromising the original runtime.\footnote{Our Python library is available at \url{https://github.com/rycolab/spanningtrees}.}
\end{abstract}

\section{Introduction}
Developing probabilistic models of dependency trees requires efficient exploration over a set of possible dependency trees, which grows exponentially with the length of the input sentence $n$.

Under an edge-factored model \citep{mcdonald-etal-2005-non,ma,dozat}, 
finding the maximum-a-posteriori dependency tree is equivalent to finding 
the maximum weight spanning tree in a weighted directed graph.
More precisely, spanning trees in \emph{directed} graphs are known as \ars. 
The maximum-weight
\ar can be found in $\bigo{n^2}$ \citep{Tarjan77, camerini1979note}.\footnote{Several authors (e.g., \citet{stanza,mcdonald-etal-2005-non}) opt for the simpler \cle algorithm \citep{chu1965shortest, bock1971algorithm, edmonds1967optimum}, which has a worst-case bound of $\bigo{n^3}$, but is often fast in practice.}

However, an oversight in the relationship between dependency trees and \ars has gone largely unnoticed in the dependency parsing literature.
Most dependency annotation standards enforce a \defn{root constraint}: Exactly one edge may emanate from the root node.\footnote{A notable exception is the Prague Dependency Treebank \cite{prague_dep}, which allows for multi-rooted trees.}
For example, the Universal Dependency Treebank (UD; \citet{ud}), a large-scale multilingual syntactic annotation effort, states in their documentation \citep{udroot}:
\begin{quote}\small
    There should be just one node with the root dependency relation in every tree.
\end{quote}

\noindent This oversight implies that parsers may return \emph{malformed} dependency trees.
Indeed, we examined the output of a state-of-the-art parser \citep{stanza} for $63$ UD treebanks.
We saw that decoding without a root constraint 
resulted in $1.80\%$ (on average) of the decoded dependency trees being malformed.
This increased to $6.21\%$ on languages that contain less than one thousand training instances with the worst case of $24\%$ on Kurmanji.

The NLP literature has proposed two solutions to enforce the root constraint: (1) Allow invalid dependency trees---hoping that the model can learn to assign them low probabilities and decode singly rooted trees,
or (2) return the best of $n$ runs of the \cle each with a fixed edge emanating from the root \citep{dozat-etal-2017-stanfords}.\footnote{In practice, if constraint violations are infrequent, this strategy should be used as a fallback for when the \emph{unconstrained} solution fails. However, this will not necessarily be the case, and is rarely the case during model training.}
The first solution is clearly problematic as it may allow parsers to predict malformed dependency trees. 
This issue is further swept under the rug with ``forgiving'' evaluation metrics, such as attachment scores, which give partial credit for malformed output.\footnote{We note exact match metrics, which consider the entire \ar, do penalize root constraint violations}
The second solution, while correct, adds an unnecessary factor of $n$ to the runtime of root-constrained decoding.

In this paper, we identify a much more efficient solution than (2).
We do so by unearthing an $\bigo{n^2}$ algorithm due to \citet{GabowT84} from the theoretical computer science literature.
This algorithm appears to have gone unnoticed in NLP literature;\footnote{There is one exception: \citet{corro-etal-2016-dependency} mention \citet{GabowT84}'s algorithm in a footnote.} we adapt the algorithm to correctly and efficiently handle the root constraint during decoding in edge-factored non-projective dependency parsing.\footnote{Much like this paper, efficient root-constrained marginal inference is also possible without picking up an extra factor of $n$, but it requires some attention to detail \citep{koo-et-al-2007, zmigrod-et-al-expectation}.}\looseness=-1

\begin{figure}[t]
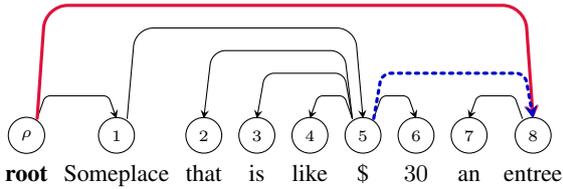

    \begin{dependency}[hide label]
    \begin{deptext}[column sep=0.1em, row sep=0ex]
        \noderoot \& \nodeid{1} \& \nodeid{2} \& \nodeid{3} \& \nodeid{4} \& \nodeid{5} \& \nodeid{6} \& \nodeid{7} \& \nodeid{8} \\ %\& \nodeid{9} \\
        \footnotesize \bf root \& \footnotesize Someplace \& \footnotesize that \& \footnotesize is \& \footnotesize like \& \footnotesize \$ \& \footnotesize $30$ \& \footnotesize an \& \footnotesize entree \& \\
    \end{deptext}
    \depedge[edge height=0.3cm]{1}{2}{}
	\depedge[edge height=0.9cm]{6}{3}{}
	\depedge[edge height=0.6cm]{6}{4}{}
	\depedge[edge height=0.3cm]{6}{5}{}
	\depedge[edge height=1.2cm]{2}{6}{}
	\depedge[edge height=0.3cm]{6}{7}{}
	\depedge[edge height=0.3cm]{9}{8}{}
	\depedge[edge height=1.5cm, edge style={Crimson, very thick}]{1}{9}{}
	\depedge[edge height=0.6cm, edge style={MediumBlue, very thick, dotted}]{6}{9}{}
    \end{dependency}
    \caption{A malformed dependency tree from our experiment. Shown are the incorrect ({\color{Crimson} highlighted}) 
    and correct ({\color{MediumBlue} highlighted})
    dependency relations for token 8.
    }
    \label{fig:deptree}
\end{figure}

\begin{figure*}[t]
\centering

\begin{tikzpicture}[baseline]
    %%%% GREEDY %%%%
    \begin{scope}[every node/.style={circle,draw, inner sep=1.5pt}]
    \node (r) at (1.3, 2.08) {\tiny $\root$};
    \node (1) at (1.1, 0.23) {\tiny $1$};
    \node (2) at (3, 1.5) {\tiny $2$};
    \node (3) at (3, 0.46) {\tiny $3$};
    \node (4) at (2.1, 0.98) {\tiny $4$};
    \end{scope}
    \begin{scope}[>=latex,
              every node/.style={fill=white, inner sep=0.5pt},
              every edge/.style={draw}]
    \path [->] (r) edge[very thick, \Cgreedy] node {\tiny $90$} (1);
    \path [->] (r) edge[bend left=20] node {\tiny $40$} (2);
    \path [->] (1) edge[bend right=25] node {\tiny $10$} (3);
    \path [->] (2) edge[very thick, \Cgreedy, bend right=50] node {\tiny $60$} (4);
    \path [->] (2) edge[] node {\tiny $30$} (3);
    \path [->] (3) edge[very thick, \Cgreedy, bend right=50] node {\tiny $50$} (2);
    \path [->] (4) edge[very thick, \Cgreedy, bend right=40] node {\tiny $70$} (3);
    \path [->] (4) edge[] node {\tiny $20$} (1);
\end{scope}
\end{tikzpicture}
\begin{tikzpicture}[baseline]
    \node[anchor=west, minimum size=0pt, inner sep=0pt] at (-0.05, 1.3) {\footnotesize \emph{(a)}};
    \node[draw, single arrow,
            minimum height=3mm, minimum width=2mm,
            single arrow head extend=0.8mm, fill=black,anchor=west, inner sep=2pt] at (0, .98) {};
\end{tikzpicture}
\begin{tikzpicture}[baseline]
    %%%% CONTRACTION %%%%
    
    % Cycle
    \begin{scope}[every node/.style={circle,draw, inner sep=1.5pt}]
    \node (r) at (1.3, 2.08) {\tiny $\root$};
    \node (1) at (1.1, 0.23) {\tiny $1$};
    \node (2) at (3, 1.5) {\tiny $2$};
    \node (3) at (3, 0.46) {\tiny $3$};
    \node (4) at (2.1, 0.98) {\tiny $4$};
    \end{scope}
    
    % Cycle node
    \node[circle, thick, draw, minimum size=2cm, \Ccyc] (c) at (2.7, 0.98) {\small$\cc$};
    \node[draw=none, minimum size=0, inner sep=0] (cr) at (2.3, 1.9) {};
    \node[draw=none, minimum size=0, inner sep=0] (c41) at (1.86, 0.43) {};
    \node[draw=none, minimum size=0, inner sep=0] (c13) at (2.5, 0.0) {};
    \node[draw=none, minimum size=0] (c3) at (1.75, 1.43) {};
    \begin{scope}[>=latex,
              every node/.style={fill=white, inner sep=0.5pt},
              every edge/.style={draw}]
    \path [->] (2) edge[\Cdead, bend right=50] node {\tiny $60$} (4);
    \path [->] (4) edge[\Cdead, bend right=40] node {\tiny $70$} (3);
    \path [->] (3) edge[\Cdead, bend right=50] node {\tiny $50$} (2);
    \path [->] (2) edge[\Cdead] node {\tiny $30$} (3);
    \path[->] (r) edge[\Center] node {\tiny $170$} (cr);
    \path [->] (cr) edge[\Center, dotted] (2);
    \path [->] (4) edge[\Cexit, dotted] (c41);
    \path [->] (c41) edge[\Cexit] node {\tiny $20$} (1);
    \path [->] (r) edge[\Cexternal] node {\tiny $90$} (1);
    \path [->] (1) edge[\Center] node {\tiny $120$} (c13);
    \path [->] (c13) edge[\Center, dotted] (3);
    \end{scope}
    
\end{tikzpicture}
\begin{tikzpicture}[baseline]
    \node[anchor=west, minimum size=0pt, inner sep=0pt] at (-0.05, 1.3) {\footnotesize \emph{(b)}};
    \node[draw, single arrow,
            minimum height=3mm, minimum width=2mm,
            single arrow head extend=0.8mm, fill=black,
            anchor=west, inner sep=2pt] at (0, .98) {};
\end{tikzpicture}
\begin{tikzpicture}[baseline]
    %%%% CONTRACTION MST %%%%
    
    % Cycle
    \begin{scope}[every node/.style={circle,draw, inner sep=1.5pt}]
    \node (r) at (1.3, 2.08) {\tiny $\root$};
    \node (1) at (1.1, 0.23) {\tiny $1$};
    \node (2) at (3, 1.5) {\tiny $2$};
    \node (3) at (3, 0.46) {\tiny $3$};
    \node (4) at (2.1, 0.98) {\tiny $4$};
    \end{scope}
    
    % Cycle node
    \node[circle, thick, draw, minimum size=2cm, \Ccyc] (c) at (2.7, 0.98) {\small$\cc$};
    \node[draw=none, minimum size=0, inner sep=0] (cr) at (2.3, 1.9) {};
    \node[draw=none, minimum size=0, inner sep=0] (c41) at (1.86, 0.43) {};
    \node[draw=none, minimum size=0, inner sep=0] (c13) at (2.5, 0.0) {};
    \node[draw=none, minimum size=0] (c3) at (1.75, 1.43) {};
    \begin{scope}[>=latex,
              every node/.style={fill=white, inner sep=0.5pt},
              every edge/.style={draw}]
    \path [->] (2) edge[bend right=50] node {\tiny $60$} (4);
    \path [->] (4) edge[bend right=40] node {\tiny $70$} (3);
    \path [->] (3) edge[bend right=50] node {\tiny $50$} (2);
    \path [->] (2) edge[] node {\tiny $30$} (3);
    \path[->] (r) edge[very thick, \Cma] node {\tiny $170$} (cr);
    \path [->] (cr) edge[dotted] (2);
    \path [->] (4) edge[dotted] (c41);
    \path [->] (c41) edge[] node {\tiny $20$} (1);
    \path [->] (r) edge[very thick, \Cma] node {\tiny $90$} (1);
    \path [->] (1) edge[] node {\tiny $120$} (c13);
    \path [->] (c13) edge[dotted] (3);
    \end{scope}
\end{tikzpicture}
\begin{tikzpicture}[baseline]
    \node[anchor=west, minimum size=0pt, inner sep=0pt] at (-0.05, 1.3) {\footnotesize \emph{(c)}};
    \node[draw, single arrow,
            minimum height=3mm, minimum width=2mm,
            single arrow head extend=0.8mm, fill=black,
            anchor=west, inner sep=2pt] at (0, .98) {};
\end{tikzpicture}
\begin{tikzpicture}[baseline]
    %%%% ROOT CONSTRAINT %%%%
    
    % Cycle
    \begin{scope}[every node/.style={circle,draw, inner sep=1.5pt}]
    \node (r) at (1.3, 2.08) {\tiny $\root$};
    \node (1) at (1.1, 0.23) {\tiny $1$};
    \node (2) at (3, 1.5) {\tiny $2$};
    \node (3) at (3, 0.46) {\tiny $3$};
    \node (4) at (2.1, 0.98) {\tiny $4$};
    \end{scope}
    
    % Cycle node
    \node[circle,thick, draw, minimum size=2cm, \Ccyc] (c) at (2.7, 0.98) {\small$\cc$};
    \node[draw=none, minimum size=0, inner sep=0] (cr) at (2.3, 1.9) {};
    \node[draw=none, minimum size=0, inner sep=0] (c41) at (1.86, 0.43) {};
    \node[draw=none, minimum size=0, inner sep=0] (c13) at (2.5, 0.0) {};
    \node[draw=none, minimum size=0] (c3) at (1.75, 1.43) {};
    \begin{scope}[>=latex,
              every node/.style={fill=white, inner sep=0.5pt},
              every edge/.style={draw}]
    \path [->] (2) edge[bend right=50] node {\tiny $60$} (4);
    \path [->] (4) edge[bend right=40] node {\tiny $70$} (3);
    \path [->] (3) edge[bend right=50] node {\tiny $50$} (2);
    \path [->] (2) edge[] node {\tiny $30$} (3);
    \path[->] (r) edge[very thin, \Cdead, dashed] node {\tiny $170$} (cr);
    \path [->] (cr) edge[very thin, \Cdead, dashed] (2);
    \path [->] (4) edge[dotted] (c41);
    \path [->] (c41) edge[] node {\tiny $20$} (1);
    \path [->] (r) edge[very thick, \Cma] node {\tiny $90$} (1);
    \path [->] (1) edge[very thick, \Cma] node {\tiny $120$} (c13);
    \path [->] (c13) edge[dotted] (3);
    \end{scope}
\end{tikzpicture}
\begin{tikzpicture}[baseline]
    \node[anchor=west, minimum size=0pt, inner sep=0pt] at (-0.05, 1.3) {\footnotesize \emph{(d)}};
    \node[draw, single arrow,
            minimum height=3mm, minimum width=2mm,
            single arrow head extend=0.8mm, fill=black,
            anchor=west, inner sep=2pt] at (0, .98) {};
\end{tikzpicture}
\begin{tikzpicture}[baseline]
    %%%% STITCH %%%%
    \begin{scope}[every node/.style={circle,draw, inner sep=1.5pt}]
    \node (r) at (1.3, 2.08) {\tiny $\root$};
    \node (1) at (1.1, 0.23) {\tiny $1$};
    \node (2) at (3, 1.5) {\tiny $2$};
    \node (3) at (3, 0.46) {\tiny $3$};
    \node (4) at (2.1, 0.98) {\tiny $4$};
    \end{scope}
    \begin{scope}[>=latex,
              every node/.style={fill=white, inner sep=0.5pt},
              every edge/.style={draw}]
    \path [->] (r) edge[very thick, \Cma] node {\tiny $90$} (1);
    \path [->] (r) edge[very thin, \Cdead, dashed, bend left=20] node {\tiny $40$} (2);
    \path [->] (1) edge[very thick, \Cma, bend right=25] node {\tiny $10$} (3);
    \path [->] (2) edge[very thick, \Cma, bend right=50] node {\tiny $60$} (4);
    \path [->] (2) edge[] node {\tiny $30$} (3);
    \path [->] (3) edge[very thick, \Cma, bend right=50] node {\tiny $50$} (2);
    \path [->] (4) edge[bend right=40] node {\tiny $70$} (3);
    \path [->] (4) edge[] node {\tiny $20$} (1);
\end{scope}
\end{tikzpicture}

\caption{
Worked example of finding the best dependency tree.
Let $G$ be the graph in the left-most figure, 
the greedy graph $\greedy{G}{}$ ({\color{\Cgreedy}highlighted}) contains a \critical $\cg$, {$\nodeId{2}\rightarrow \nodeId{4} \rightarrow \nodeId{3} \rightarrow \nodeId{2}$}.
%%%%
Step \emph{(a)} shows the contraction $\contract{G}{\cg}$ where $\cg$ is replaced by {\color{\Ccyc}\nodeId{\cc}}, and edges are cast as \Wenter, \Wexit, \Wexternal, or \Wdead edges in $\contract{G}{\cg}$.
We see the bookkeeping function $\prov$ (as \dottedarrow), e.g., $\prov\edge{\cc}{20}{1} = \edge{4}{20}{1}$ and $\prov\edge{\root}{170}{\cc} = \edge{\root}{40}{2}$.
%%%%
Step \emph{(b)} takes the greedy (sub)graph of $\contract{G}{\cg}$ and since it contains no cycles, it is $\mst{(\contract{G}{\cg})}{}$ as ({\color{\Cma}highlighted}).
Note that if we did not require a dependency tree, we could now use \cref{thm:ma} to break $\cg$ at \nodeId{2}.
%%%
Step \emph{(c)} takes $\mst{(\contract{G}{\cg})}{}$, which has \emph{two} root edges, $\edge{\root}{90}{1}$ and $\edge{\root}{170}{\cc}$, and removes the edge with minimal consequence: 
removing $\edge{\root}{90}{1}$ leads to $\bar{w}=190$, while 
removing $\edge{\root}{170}{\cc}$ leads to $\bar{w}=210$.  
We pick the latter.
As deleting $\edge{\root}{170}{\cc}$ does not lead to a critical cycle (\emph{optimization case}), we remove it from the graph (shown as \dashedarrow) and so we get $\bestDep{(\contract{G}{\cg})}$ ({\color{\Cma}highlighted}).
%%%%%
Step \emph{(d)} stitches $\stitch{\bestDep{(\contract{G}{\cg})}{}}{\cg^{(3)}}$ yielding $\bestDep{G}$ ({\color{\Cma}highlighted}).}
\label{fig:example}
\end{figure*}
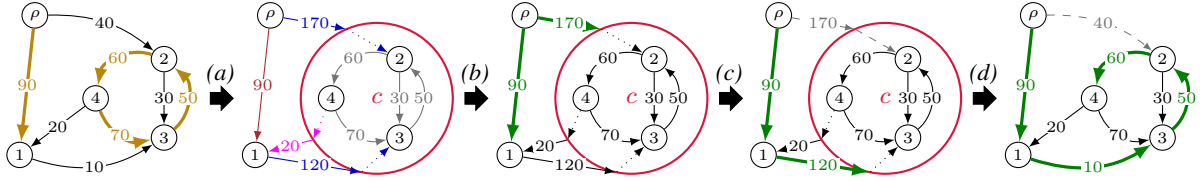%

\section{Approach}
In this section, the marker {\smaller{$\figrefmarker$}} indicates that a recently introduced concept is illustrated the worked example in \cref{fig:example}.
Let $G = (\root, V, E)$ be a \defn{rooted weighted directed graph}
where 
$V$ is a set of nodes,
$E$ is a set of weighted edges,~$E \subseteq \{ \edge{i}{w}{j} \mid i,j \in V,\, w \in \mathbb{R} \}$,\footnote{When there is no ambiguity, we may abuse notation using $G$ to refer to either its node or edge set, 
e.g., we may write ${\edge{i}{}{j}\in G}$ to mean ${\edge{i}{}{j}\in E}$, and ${i\in G}$ to mean ${i\in V}$.}
and $\root \in V$ is a designated root node with no incoming edges.
In terms of dependency parsing, each non-$\root$ node corresponds to a token in the sentence, and $\root$ represents the special root token that is not a token in the sentence.
Edges represent possible dependency relations between tokens.  The edge weights are scores from a model (e.g., linear \cite{mcdonald-etal-2005-non}, or neural network \cite{dozat-etal-2017-stanfords}).
\cref{fig:deptree} shows an example. 
We allow $G$ to be a \defn{multi-graph}, i.e., we allow multiple edges between pairs of nodes.
Multi-graphs are a natural encoding of \emph{labeled} dependency relations where possible labels between words are captured by multiple edges between nodes in the graph.
Multi-graphs pose no difficulty as only the highest-weight edge between two nodes may be selected in the returned tree.

\newcommand{\constraint}[1]{{\small({#1})}\xspace}
\newcommand{\arbIncoming}[0]{\constraint{C1}}
\newcommand{\arbNoCycles}[0]{\constraint{C2}}
An \defn{\ar} of $G$ is a subgraph $\tree = (\root, V, E')$
where $E' \subseteq E$ such that:
\begin{enumerate}[leftmargin=2.5em, topsep=8pt, itemsep=1pt, parsep=5pt]
\item[\arbIncoming] Each non-root node has exactly one incoming edge (thus, ${\abs{E'} = \abs{V}{-}1}$); \label{arb:incoming}
\item[\arbNoCycles] $\tree$ has no cycles. \label{arb:nocycles}
\end{enumerate}
\newcommand{\depRC}[0]{\constraint{C3}}
A \defn{dependency tree} of $G$ is an \ar that additionally satisfies 
\begin{enumerate}[leftmargin=2.5em, topsep=8pt, itemsep=1pt, parsep=1pt]
    \item[\depRC] $\abs{\{ \edge{\root}{}{\_} \in E' \} } = 1$
\end{enumerate}
In words, \depRC says $\tree$ contains exactly one out-edge from $\root$.
Let $\arbs{}{G}$ and $\arbsone{}{G}$ denote the sets of \ars and dependency trees, respectively.\looseness=-1

The weight of a graph or subgraph is defined as
\begin{equation}\label{eq:additive-weight}
    \treecost{G} \defeq \smashoperator{\sum_{\edge{i}{w}{j} \in G}}\, w
\end{equation}
\noindent In \cref{sec:ma}, we describe an efficient algorithm for finding the best (highest-weight) \ar 
\begin{equation}
   \mst{G}{} = \argmax_{\tree \in \arbs{}{G}} \treecost{\tree} \\
\end{equation}
and, in \cref{sec:root-constraint}, the best dependency tree.\footnote{Probabilistic models of \ars
(e.g., \citet{koo-et-al-2007,dozat})
typically seek the maximum a posteriori structure, $\argmax_{\!\!\tree} \prod_{e \in \tree} p_e$
$=$ $\argmax_{\!\!\tree} \sum_{e \in \tree} \log p_e$. 
This case can be solved as \cref{eq:additive-weight} by taking the weight of $e$ to be $\log p_e$ because $p_e \ge 0$.}%
\begin{equation}
   \bestDep{G} = \argmax_{\tree \in \arbsone{}{G}} \treecost{\tree}
\end{equation}

\subsection{Finding the best \ar}\label{sec:ma}
A first stab at finding $\mst{G}{}$ would be to select the best (non-self-loop) incoming edge for each node.  Although, this satisfies \arbIncoming, it does not (necessarily) satisfy \arbNoCycles.
We call this subgraph the \defn{greedy graph}, denoted $\greedy{G}{}$.\figref
Clearly, $\treecost{\greedy{G}{}} \ge \treecost{\mst{G}{}}$ since it is subject to fewer restrictions.  Furthermore, if $\greedy{G}{}$ happens to be acyclic, it is clearly equal to $\mst{G}{}$.
What are we to do in the event of a cycle? That answer has two parts.

\emph{Part 1:} We call any cycle $\cg$ in $\greedy{G}{}$ a \defn{critical cycle}.\figref
Naturally, \arbNoCycles implies that \criticals can never be part of an \ar.
However, they help us identify optimal \ars for certain \emph{subproblems}.  Specifically, if we were to ``break'' the cycle at any node $j \in \cg$ by removing its (unique) incoming edge, we would have an optimal \ar rooted at $j$ for the subgraph over the nodes in $\cg$.  
Let $\cg^{(j)}$ be a subgraph of $\cg$ rooted at $j$ that denotes the broken cycle at $j$.
Let $G^{(j)}_\cg$ be the subgraph rooted at $j$ where $G_\cg$ contains all the nodes in $\cg$ and all edges between them from $G$.
Since $\cg$ is a critical cycle,  $\cg^{(j)}$ is the greedy graph of ${G^{(j)}_\cg}$. Moreover, as it is acyclic, we have that {$\cg^{(j)} = \mst{({G^{(j)}_{\cg}})}{}$}.
The key to finding the best \ar of the entire graph is, thus, determining where to break \criticals.

\emph{Part 2:} Breaking cycles is done with a recursive algorithm that solves the ``outer problem'' of fitting the (unbroken) cycle into an optimal \ar.  The algorithm treats the cycle as a single \emph{contracted} node.  Formally, a \defn{cycle contraction} takes a graph $G$ and a (not necessarily critical) cycle $\cg$, 
and creates a new graph denoted $\contract{G}{\cg}$ with
the same root,
nodes $(V\setminus\cg\cup\{\cc\})$
where $\cc \notin V$ is a new node that represents the cycle,
and contains the following set of edges:
For any $\edge{i}{w}{j}\in G$
\begin{itemize}[leftmargin=1em, topsep=8pt, itemsep=1pt, parsep=5pt]
    \item \Wenter: if $i\!\notin\!\cg, j\!\in\!\cg$, then  $\edge{i}{w'}{\cc}\in \contract{G}{\cg}$ where $w' = w + \treecost{\cg^{(j)}}$. Akin to dynamic programming, this choice edge weight (due to \citet{georgiadis}) gives the best ``cost-to-go'' for breaking the cycle at $j$.
    \item \Wexit: if $i\!\in\!\cg, j\!\notin\!\cg$, then $\edge{\cc}{w}{j}\in \contract{G}{\cg}$
    \item \Wexternal: if $i\!\notin\!\cg, j\!\notin\!\cg$, then  $\edge{i}{w}{j}\in \contract{G}{\cg}$
    \item \Wdead: if $i\!\in\!\cg, j\!\in\!\cg$, then no edge related to $\edge{i}{w}{j}$ is in $\contract{G}{\cg}$. This is because such an edge $\edge{\cc}{}{\cc}$ would be a self-cycle, which can never be part of an \ar.    
\end{itemize}

\vspace{\baselineskip}
\noindent Additionally, we define a \defn{bookkeeping function}, $\prov$, which maps the nodes and edges of $\contract{G}{\cg}$ to their counterparts in $G$. We overload $\prov(G)$ to apply point-wise to the constituent nodes and edges.\figref

By \arbIncoming, we have that for any $\tree_{\cg}\in\arbs{}{\contract{G}{\cg}}$, there exists exactly one incoming edge $\edge{i}{}{\cc}$ to the cycle node $\cc$.
We can use $\prov$ to infer where the cycle was broken with $\prov\edge{i}{}{\cc}=\edge{i}{}{j}$.
We call $j$ the \textbf{\entersite} of $\tree_{\cg}$.
Consequently, we can stitch together an \ar as $\prov(\tree_{\cg}) \cup \cg^{(j)}$. We use the shorthand $\stitch{\tree_{\cg}}{\cg^{(j)}}$ for this operation due to its visual similarity to unraveling a cycle.\figref

$\contract{G}{\cg}$ may also have a \critical, so we have to apply this reasoning recursively.
This is captured by \citet{Karp71}'s Theorem 1.\footnote{We have lightly modified the original theorem.  For completeness, \cref{app:karp} provides a proof in our notation.}

\begin{thm}\label{thm:ma}
For any graph $G$, either $\mst{G}{}=\greedy{G}{}$ or $G$ contains a \critical $\cg$ and $\mst{G}{}=\stitch{\mst{(\contract{G}{\cg})}{}}{\cg^{(j)}}$ where $j$ is the \entersite of $\mst{(\contract{G}{\cg})}{}$.
Furthermore, $\treecost{\mst{(\contract{G}{\cg})}{}} = \treecost{\mst{G}{}}$.
\end{thm}

\cref{thm:ma} suggests a recursive strategy for finding $\mst{G}{}$, which is the basis of many efficient algorithms \citep{Tarjan77, camerini1979note, georgiadis,chu1965shortest,bock1971algorithm,edmonds1967optimum}.
We detail one such algorithm in \cref{alg:ma}.
\cref{alg:ma} can be made to run in $\bigo{n^2}$ time for dense
with the appropriate implementation choices, such as 
Union-Find \citep{unionfind} to maintain membership of nodes to contracted nodes, as well as 
radix sort \cite{radix} to sort incoming edges to contracted nodes; using a regular sort would add a factor of $\log n$ to the runtime.

\begin{algorithm}[t]
\small
\begin{algorithmic}[1]
\Func{$\opt{G}{}$} \Comment{Find $\mst{G}{} \in \arbs{}{G}$ or $\bestDep{G}\in\arbsone{}{G}$}

\If{$\greedy{G}{}$ has a cycle $\cg$}
\Comment{Recursive case}
\State \Return $\stitch{\opt{\contract{G}{\cg}}{}}{\cg^{(j)}}$
\Else \Comment{Base case}
\If{we require a dependency tree (\protect\cref{sec:root-constraint})} \label{line:if}
\State \Return $\constrain{G}{}$
\Else
\State \Return $\greedy{G}{}$
\EndIf
\EndIf
\EndFunc

\vspace{.5\baselineskip}

\Func{$\constrain{G}{}$} \Comment{
Find $\bestDep{G} \in \arbsone{}{G}$; $\greedy{G}{} \in \arbs{}{G}$.
}
\State $\sigma \gets$ set of $\root$'s outgoing edges in $\greedy{G}{}$
\If{$\abs{\sigma} = 1$}
\Return $\greedy{G}{}$ \Comment{Root constraint satisfied}
\EndIf 

\State $G' \gets \smashoperator{\argmax\limits_{e\in\sigma: G''=\multidelete{G}{e}}}\treecost{\greedy{G''}{}}$  \Comment{Find best edge removal}
\If{$\greedy{G'}{}$ has cycle $\cg$} \Comment{Reduction case}
\State \Return $\stitch{\constrain{\contract{G}{\cg}}{}}{\cg^{(j)}}$
\Else \Comment{Optimization case}
\State \Return $\constrain{G'}{}$
\EndIf
\EndFunc
\end{algorithmic}
\caption{}
\label{alg:ma}
\end{algorithm}

\subsection{Finding the best dependency tree}
\label{sec:root-constraint}
\citet{GabowT84} propose an algorithm that does additional recursion at the base case of $\opt{G}{}$ (the additional if-statement at \cref{line:if}) to recover $\bestDep{G}$ instead of $\mst{G}{}$.

Suppose that the set of edges emanating from the root in $\greedy{G}{}$ is given by $\sigma$ and $\abs{\sigma}> 1$.
We consider removing each edge in $\edge{\root}{}{j}\in\sigma$ from $G$.
Since $G$ may have \emph{multiple} edges from $\root$ to $j$, we write $\multidelete{G}{e}$ to mean deleting \emph{all} edges with the same edge points as $e$.
Let $G'$ be the graph $\multidelete{G}{e'}$ where $e'\in\sigma$ is chosen greedily to maximize $\treecost{\greedy{G'}{}}$.  Consider the two possible cases:

\emph{Optimization case.} If $G'$ has no \criticals, then $\greedy{G'}{}$  must be the best \ar with one fewer edges emanating from the root than $\greedy{G}{}$ by our greedy choice of $e'$.\figref

\emph{Reduction case.} If $G'$ has a \critical $\cg$,
then all edges in $\cg$ that do not point to $j$ are in $\greedy{G}{}$.
If $e'\notin\bestDep{G}$, then $\cg$ is \critical in the context of constrained problem and so we can apply \cref{thm:ma} to recover $\bestDep{G}$.
Otherwise, $e\in\bestDep{G}$ and we can break $\cg$ at $j$ to get $\cg^{(j)}$, which is comprised of edges in $\greedy{G}{}$.
Therefore, we can find $\bestDep{(\contract{G}{\cg})}$ to retrieve $\bestDep{G}$.
This notion is formalized in the following theorem.\footnote{For completeness, \cref{app:root} provides a proof of \cref{thm:root}.}\looseness=-1

\begin{thm}\label{thm:root}
For any graph $G$ with $\mst{G}{}\!=\!\greedy{G}{}$, let $\sigma$ be the set of outgoing edges from $\root$ in $\mst{G}{}$.
If $\abs{\sigma}\!=\!1$, then $\bestDep{G}\!=\!\mst{G}{}$.
Otherwise, let $G'\!=\!\multidelete{G}{e'}$ for $e'\in\sigma$ that maximizes $\treecost{\greedy{G'}{}}$,
then either $\bestDep{G}\!=\!\bestDep{G'}$ or there exists a \critical $\cg$ in $G'$ such that $\bestDep{G}\!=\!\stitch{\bestDep{(\contract{G}{\cg})}}{\cg^{(j)}}$ where $j$ is the \entersite of $\bestDep{(\contract{G}{\cg})}$.
\end{thm}

\cref{thm:root} suggests a recursive strategy $\mathrm{constrain}$ (\cref{alg:ma}) for finding $\bestDep{G}$ given $\mst{G}{}$.
\citet[Theorem 7.1]{GabowT84} prove that such a strategy will execute in $\bigo{n^2}$ and so when combined with $\opt{G}{}$ (\cref{alg:ma}) leads to a $\bigo{n^2}$ runtime for finding $\bestDep{G}$ given a graph $G$.
The efficiency of the algorithm amounts to requiring a bound of $\bigo{n}$ calls to $\mathrm{constrain}$ that will lead to the \emph{reduction case} in order to obtain any number \emph{optimization cases}.
Each recursive call does a linear amount of work to search for the edge to remove and to stitch together the results of recursion.
Rather than computing the greedy graph from scratch, 
implementations should exploit that each edge removal will only change one element of the greedy graph. Thus, we can find $\treecost{\greedy{\multidelete{G}{e'}}{}}$ in constant time.

\begin{table*}[ht]
    \centering
    \begin{tabular}{lcccc}
         \bf Setting & \bf \# Languages & \bf Malformed rate & \bf Rel. $\Delta$ UAS & \bf Rel. $\Delta$ Exact Match \\ \midrule
         High & $20$ & $0.63\%$ & $0.0041\%$ & $0.15\%$  \\
         Medium & $32$ & $1.02\%$ & $0.0012\%$ & $0.22\%$\\
         Low & $11$ & $6.21\%$ & $0.0368\%$ & $2.91\%$
    \end{tabular}
    \caption{Average malformed rate, relative UAS change, and relative exact match score change for different data settings. The 63 languages are split by their training set size $\abs{\mathrm{train}}$ into high ($\abs{\mathrm{train}}\ge 10,000$), medium ($1,000\le\abs{\mathrm{train}}<10,000$), and low ($\abs{\mathrm{train}}<1,000$).
    }
    \label{tab:breakdown}
\end{table*}

\begin{figure}[t]
    \centering
    \includegraphics[width=0.45\textwidth]{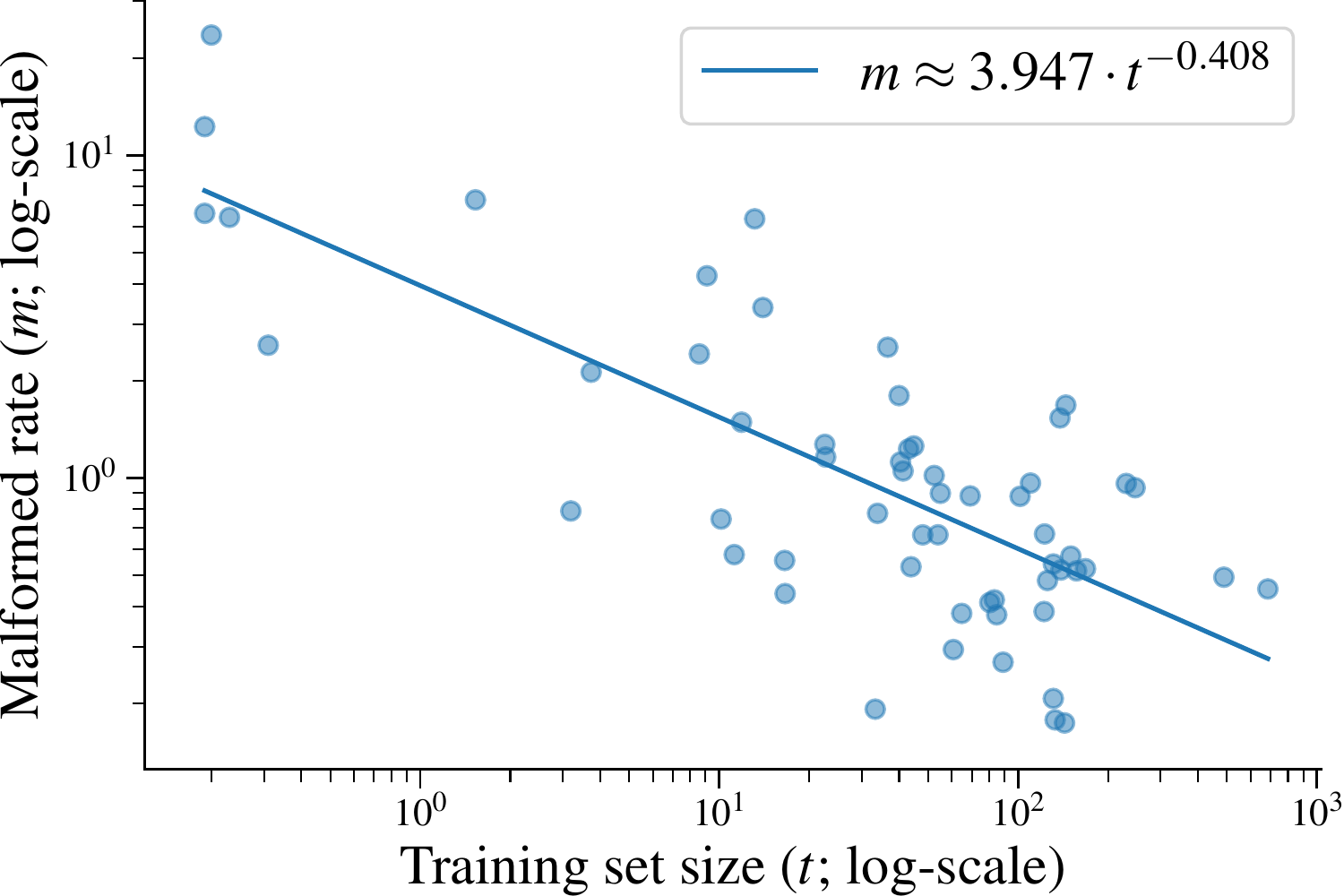}
    \caption{Proportion of malformed trees when decoding pre-trained models \citep{stanza} for languages with varying training set sizes.}
    \label{fig:malformed}
\end{figure}

\begin{figure*}[!ht]
    \centering
    \begin{minipage}{0.47\textwidth}
        \centering
        \includegraphics[width=0.98\textwidth]{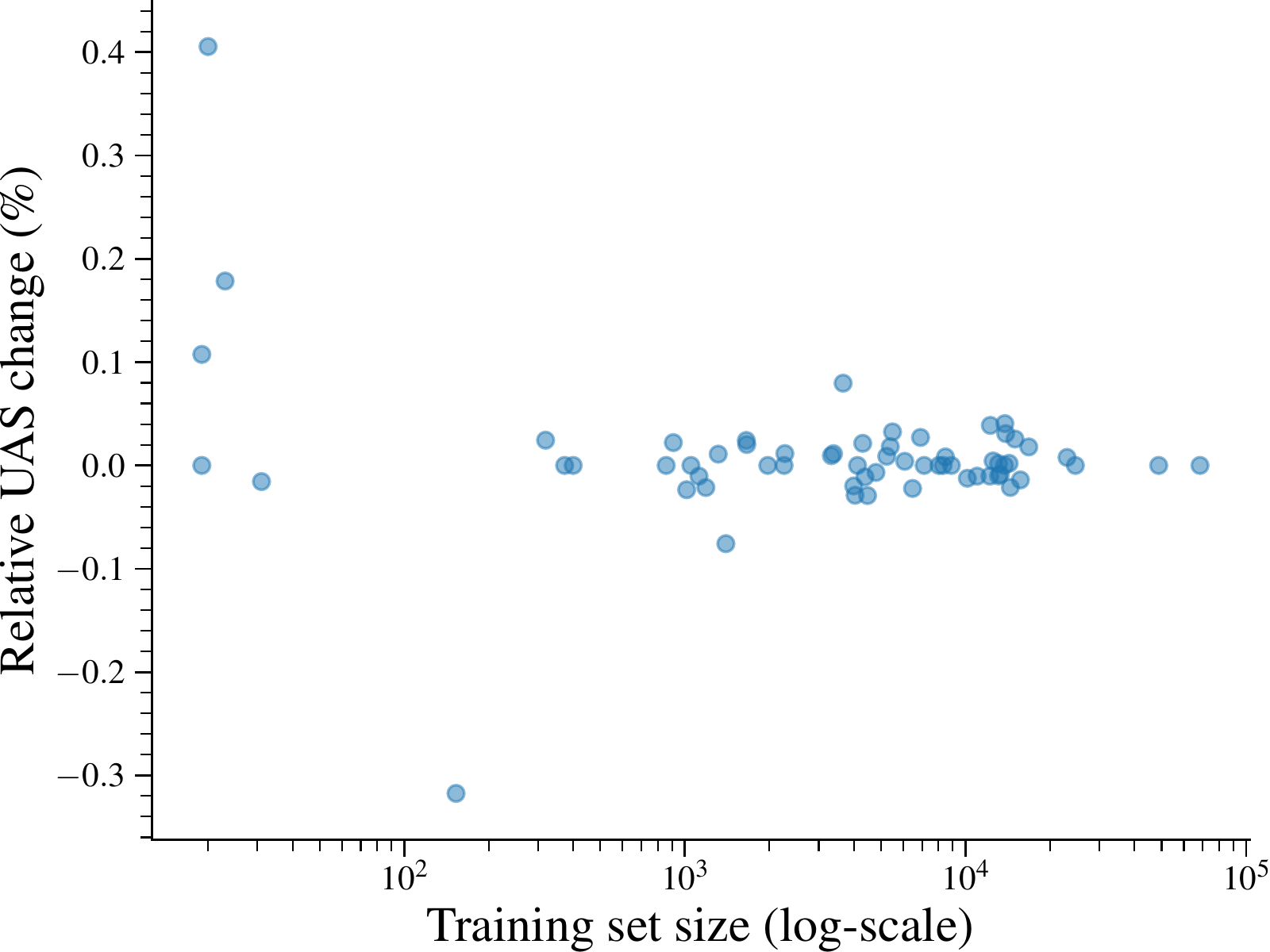}
    \end{minipage}\hfill
    \begin{minipage}{0.47\textwidth}
        \centering
        \includegraphics[width=0.98\textwidth]{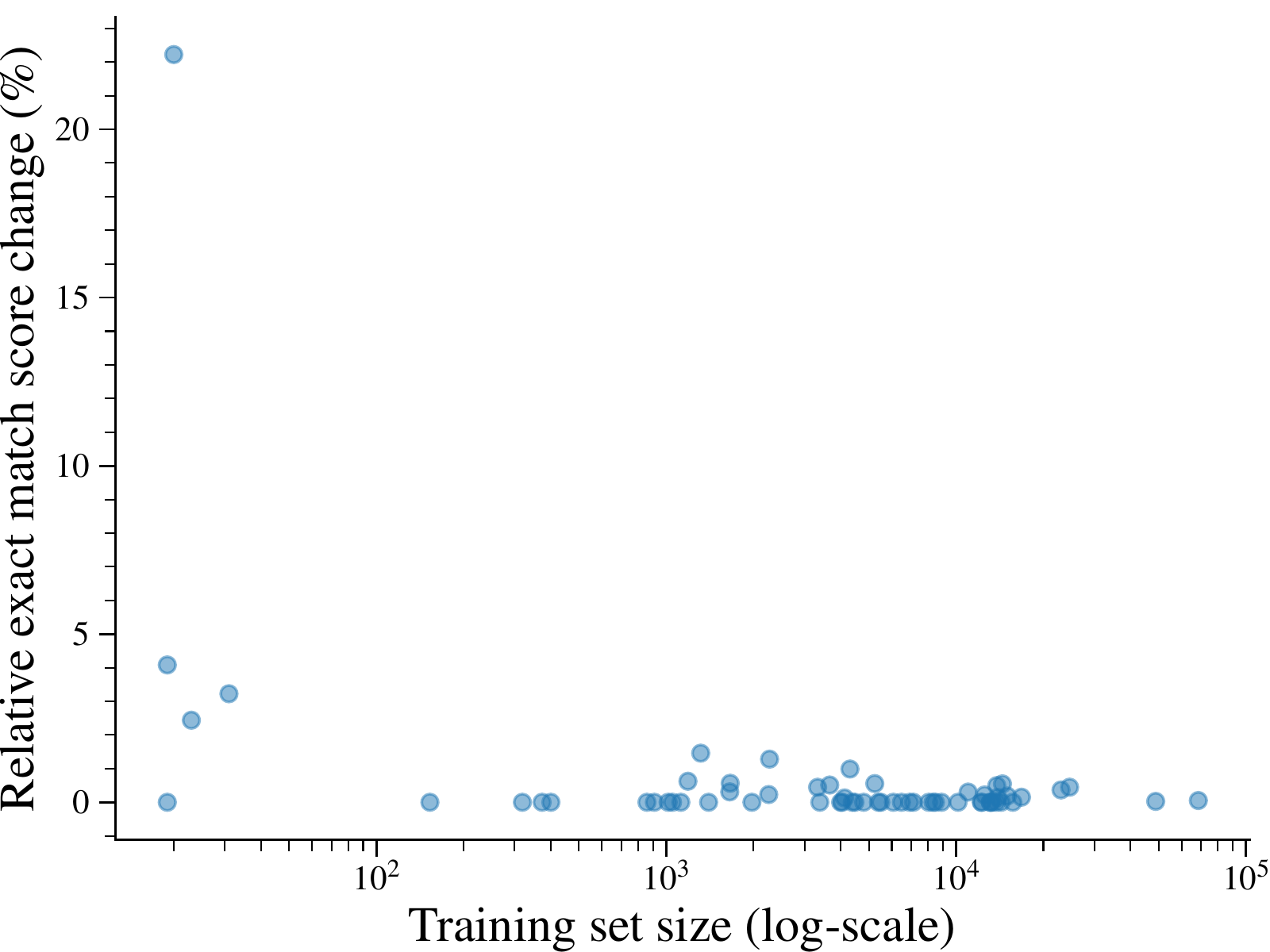}
    \end{minipage}
    \caption{Relative change in UAS and exact match score when using the unconstrained and constrained algorithms for languages with varying training set sizes.}
    \label{fig:change} 
\end{figure*}

\section{Experiment}\label{sec:experiment}

How often do state-of-the-art parsers generate malformed dependency trees?
We examined $63$ Universal Dependency Treebanks \citep{ud} and computed the rate of malformed trees when decoding using edge weights generated by pre-trained models supplied by \citet{stanza}.
On average, we observed that $1.80\%$ of trees are malformed.
We were surprised to see that---although the edge-factored model used is not expressive enough to capture the root constraint \emph{exactly}---there are useful correlates of the root constraint in the surface form of the sentence, which the model appears to use to workaround this limitation.
This becomes further evident when we examine the relative change\footnote{The relative difference is computed with respect to the unconstrained algorithm's scores.} in UAS ($0.0083\%$) and exact match scores ($0.60\%$) when using the constrained algorithm as opposed to the unconstrained algorithm.

Nevertheless, given less data, it is harder to learn to exploit the surface correlates; thus, we see an increasing average rate of violation, $6.21\%$, when examining languages with training set sizes of less than $1,000$ sentences.
Similarly, the relative change in UAS and exact match score increases to $0.0368\%$ and $2.91\%$ respectively.
Indeed, the worst violation rate was $24\%$ was seen for Kurmanji which only contains $20$ sentences in the training set.
Kurmanji consequently had the largest relative changes to both UAS and exact match scores of $0.41\%$ and $22.22\%$.
We break down the malformed rate and accuracy changes by training size in \cref{tab:breakdown}.
Furthermore, the correlation between training size and malformed tree rate can be seen in \cref{fig:malformed} while the correlation between training size and relative accuracy change can be seen in \cref{fig:change}.
We provide a full table of the results in \cref{app:table}.

\section{Conclusion}
In this paper, we have bridged the gap between the graph-theory and dependency parsing literature.
We presented an efficient $\bigo{n^2}$ for finding the maximum \ar of a graph.
Furthermore, we highlighted an important distinction between dependency trees and \ars, namely that dependency trees are \ars subject to a \emph{root constraint}.
Previous work uses inefficient algorithms to enforce this constraint.
We provide a solution which runs in $\bigo{n^2}$.
Our hope is that this paper will remind future research in dependency parsing to please mind the root.

\section*{Acknowledgments}
We would like to thank all reviewers for their valuable feedback and suggestions. The first author is supported by the University of Cambridge School of Technology Vice-Chancellor's Scholarship as well as by the University of Cambridge Department of Computer Science and Technology's EPSRC.

\bibliography{acl2020}
\bibliographystyle{acl_natbib}

\clearpage

\appendix

\onecolumn

\section{Proof of \cref{thm:ma}}\label{app:karp}
To prove \cref{thm:ma}, we note a correspondence between graphs and contracted graphs.

\begin{prop}\label{prop:GtoContract}
Given a rooted graph $G$ and a (not necessarily critical) cycle $\cg$ in $G$.
For any $\tree \in \arbs{}{G}$ that has a single edge $e=\edge{i}{w}{j}\in\tree$ such that $i\notin \cg$ and $j\in \cg$,
there exists $\tree_{\cg}\in\arbs{}{\contract{G}{\cg}}$ and $\treep\in\arbs{}{G^{(j)}_{\cg}}$ such that $\tree = \stitch{\tree_{\cg}}{\treep}$.
Furthermore,
\begin{equation}\label{eq:treecost}
    \treecost{\tree}= \treecost{\tree_{\cg}} - \treecost{\cg^{(j)}} + \treecost{\treep}
\end{equation}
\end{prop}
\begin{proof}
Since $e$ is the only edge in $\tree$ from a non-cycle node to a cycle node (\Wenter), every edge $e'\in\contract{G}{\cg}$ such that $\prov(e')\in\tree$ forms an \ar $\tree_{\cg}\in\arbs{}{\contract{G}{\cg}}$.
Note that the set of edges in $\tree$ for which there is no corresponding edge in $\contract{G}{\cg}$ are \Wdead edges.
In fact, as $\tree$ satisfies \arbIncoming, these edges form an \ar $\treep\in\arbs{}{G^{(j)}_{\cg}}$.
Therefore, $\tree = \stitch{\tree_{\cg}}{\treep}$.

\noindent Furthermore, consider the weight of $\tree$:
\begin{align}
    \treecost{\tree} &= \smashoperator{\sum_{\edge{i'}{w'}{j'}\in\prov\left(\tree_{\cg}\right)}}w' + \treecost{\treep} \\
    &= \smashoperator{\sum_{\edge{i'}{w'}{j'}\in\prov\left(\tree_{\cg}\setminus\{e\}\right)}}w' + w + \treecost{\treep} \\
    &= \smashoperator{\sum_{\edge{i'}{w'}{j'}\in\tree_{\cg}\setminus\{e\}}}w' + w + \treecost{\treep} \label{eq:enter} \\
    &= \smashoperator{\sum_{\edge{i'}{w'}{j'}\in\tree_{\cg}}}w' - \treecost{\cg^{(j)}} + \treecost{\treep} \label{eq:subst} \\
    &= \treecost{\tree_{\cg}} - \treecost{\cg^{(j)}} + \treecost{\treep}
\end{align}
Note that \cref{eq:enter} follows because $e$ is the only edge in $\tree$ from a non-cycle node to a cycle node, and \cref{eq:subst} follows by the construction of \Wenter edges in $\contract{G}{\cg}$.
\end{proof}

As a corollary, we also have that every \ar in the contracted graph $\contract{G}{\cg}$ can be expanded into an \ar in $G$. 

\begin{cor}[Expansion lemma]\label{cor:ContracttoG}
Given a rooted graph $G$ with a cycle $\cg$, every \ar $\tree_{\cg} \in \arbs{}{\contract{G}{\cg}}$ is related to an \ar $\tree \in \arbs{}{G}$ by $\tree = \stitch{\tree_{\cg}}{\cg^{(j)}}$ where $j$ is the \entersite of $\tree_{\cg}$.
Furthermore $\treecost{\tree}=\treecost{\tree_{\cg}}$.
\end{cor}
\begin{proof}
Let $j$ be the \entersite of $\tree_{\cg}$ into $\cg$.
As $\tree_{\cg} \in \arbs{}{\contract{G}{\cg}}$ and $\cg^{(j)}\in\arbs{}{G_\cg^{(j)}}$, \cref{prop:GtoContract} constructs $\tree\in\arbs{\root}{G}$ as desired.
Furthermore, $\treecost{\tree}=\treecost{\tree_{\cg}}-\treecost{\cg^{(j)}}+\treecost{\cg^{(j)}}=\treecost{\tree_{\cg}}$.
\end{proof}

Note that \cref{prop:GtoContract} does not account for all \ars in $\arbs{}{G}$.
We next show that such \ars which cannot be constructed using \cref{prop:GtoContract} will never be $\mst{G}{}$.

\begin{lemma}\label{lemma:c-j}
Given a rooted graph $G$ with a \critical $\cg$.  We have that for all $j\in\cg$
\begin{equation}
   \mst{{G^{(j)}_{\cg}}}{} = \cg^{(j)}
\end{equation}
\end{lemma}
\begin{proof}
Since $G^{(j)}_{\cg}$ is a subgraph of $G$ it must be that $\greedy{{G^{(j)}_{\cg}}}{}$ is also a subgraph of $\greedy{G}{}$.
Since $\cg$ is a \critical, $\cg^{(j)}$ does not have cycles and equals $\greedy{{G^{(j)}_{\cg}}}{}$.
Therefore $\cg^{(j)}=\mst{{G^{(j)}_{\cg}}}{}$.
\end{proof}

\begin{lemma}\label{lemma:enter}
Given a rooted graph $G$ with a \critical $\cg$ and $\tree\in\arbs{}{G}$.
If $e=\edge{i}{}{j}\in\tree$ and $e'=\edge{i'}{}{j'}$ such that $i,i'\notin \cg$ and $j,j'\in \cg$,
then there exists a $\treep\in\arbs{}{G}$ with $e\in\treep$ and $e'\notin\treep$ such that $\treecost{\tree}\leq\treecost{\treep}$.
\end{lemma}
\begin{proof}
Construct $\treep$ such that for every edge $e''=\edge{i''}{}{j''}\in\contract{G}{\cg}$, if $j''\neq \cc$ and $\prov(e'')\in\tree$, then $\prov(e'')\in\treep$.
Additionally, let $e$ be in $\treep$ as well as the edges in $\cg^{(j)}$.
Then $\treep$ has no cycles and each non-root node contains a single incoming edge, so $\treep\in\arbs{}{G}$.
Since $\tree$ and $\treep$ contain identical edges except for those pointing to nodes in $\cg\setminus\{j\}$, by \cref{lemma:c-j}, $\treecost{\tree}\leq\treecost{\treep}$.
\end{proof}

\begin{customthm}{1}
For any graph $G$, either $\mst{G}{}=\greedy{G}{}$ or $G$ contains a \critical $\cg$ and $\mst{G}{}=\stitch{\mst{(\contract{G}{\cg})}{}}{\cg^{(j)}}$ where $j$ is the \entersite of $\mst{(\contract{G}{\cg})}{}$.
Furthermore, $\treecost{\mst{{\contract{G}{\cg}}}{}} = \treecost{\mst{G}{}}$.
\end{customthm}
\begin{proof} There are two cases to consider.

\emph{Case 1}: $G$ does not contain a \critical. Trivially, $\mst{G}{} = \greedy{G}{}$.

\emph{Case 2}: $G$ contains a \critical $\cg$.
By \cref{cor:ContracttoG}, we can construct an \ar $\tree=\stitch{\mst{(\contract{G}{\cg})}{}}{\cg^{(j)}}\in\arbs{}{G}$, we now prove that no other $\treep\in\arbs{}{G}$ can have a higher weight.
Firstly, by \cref{lemma:enter}, we only need to consider $\treep$ that satisfy \cref{prop:GtoContract}.
Therefore, $\treep$ must be decomposable into an \ar $\tree_{\cg}\in\arbs{}{\contract{G}{\cg}}$ and an \ar in $\arbs{}{G_{\cg}^{(j')}}$ where $j'$ is the \entersite of $\tree_{\cg}$.
Then since $\mst{(\contract{G}{\cg})}{}$ is optimal, we have that $\tree_{\cg}=\mst{(\contract{G}{\cg})}{}$ and $j'=j$.
As $\cg^{(j)}$ is optimal (by \cref{lemma:c-j}), $\tree$ must also be optimal and so $\mst{G}{}=\stitch{\mst{(\contract{G}{\cg})}{}}{\cg^{(j)}}$.
\end{proof}

\newpage

\section{Proof of \cref{thm:root}}\label{app:root}
We prove \cref{thm:root} by showing that both the \emph{optimization} and \emph{reduction} cases described in the main text lead to progress towards finding $\bestDep{G}$.
\begin{lemma}\label{lemma:root-reduce}
For any graph $G$ with $\mst{G}{}=\greedy{G}{}$, let $\sigma$ be the set of outgoing edges from $\root$ in $\greedy{G}{}$.
If $\abs{\sigma}\!>\!1$, let $G'\!=\!\multidelete{G}{e'}$ for $e'\in\sigma$ that maximizes $\treecost{\greedy{G'}{}}$.
If there exists a \critical $\cg$ in $G'$, then $\bestDep{G}\!=\!\stitch{\bestDep{(\contract{G}{\cg})}}{\cg^{(j)}}$ where $j$ is the \entersite of $\bestDep{(\contract{G}{\cg})}$.\looseness=-1
\end{lemma}
\begin{proof}
Let $e'=\edge{\root}{}{i}$ and $e\in\contract{G}{\cg}$ such that $\prov(e)=e'$. We know that $e$ always exists as $e'$ emanates from the root.
By \cref{cor:ContracttoG}, we know that $\tree=\stitch{\bestDep{(\contract{G}{\cg})}}{\cg^{(j)}}\in\arbs{}{G}$ where $j$ is the \entersite of $\bestDep{(\contract{G}{\cg})}$.
Furthermore, As $\cg$ has no edges emanating from the root, $\tree\in\arbsone{}{G}$.
There are two cases to consider:

\emph{Case 1} ($e\in\bestDep{(\contract{G}{\cg})}$):
As $\cg^{(j)}$ is a subgraph of $\greedy{G}{}$, $\tree$ must have the highest weight in $\arbsone{}{G}$, so $\bestDep{G}=\tree$.

\emph{Case 2} ($e\notin\bestDep{(\contract{G}{\cg})}$):
Then $e'$ cannot be in $\bestDep{G}$, and the edge pointing to $i$ in $\cg$ is the next best possible edge incoming to $j$.
Therefore, whichever way we break $\cg$ in $\tree$, we will get a set of edges with maximal weight and so $\bestDep{G}=\tree$.
\end{proof}

\begin{lemma}\label{lemma:root}
For any graph $G$ with $\mst{G}{}=\greedy{G}{}$, let $\sigma$ be the set of outgoing edges from $\root$ in $\greedy{G}{}$.
If $\abs{\sigma}\!>\!1$, let $G'\!=\!\multidelete{G}{e'}$ for $e'\in\sigma$ that maximizes $\treecost{\greedy{G'}{}}$.
Either $\bestDep{G}\!=\!\bestDep{G'}$ or there exists a \critical $\cg$ in $G'$ such that $\bestDep{G}\!=\!\stitch{\bestDep{(\contract{G}{\cg})}}{\cg^{(j)}}$ where $j$ is the \entersite of $\bestDep{(\contract{G}{\cg})}$.
\end{lemma}
\begin{proof}
Let $j$ be the \entersite of $\bestDep{(\contract{G}{\cg})}$.
Proof by induction on $r = \abs{\sigma}$.

\emph{Base case} ($r=2$): If $G'$ does not contain a \critical, then clearly $\bestDep{G'}=\mst{G'}{}$. Since we choose $e'$ to maximize $\greedy{G'}{}$ and $G'$ is a subgraph of $G$, $\bestDep{G}=\bestDep{G'}$.
Otherwise, $G'$ has a \critical $\cg$.
Then by \cref{lemma:root-reduce}, $\bestDep{G}=\stitch{\bestDep{(\contract{G}{\cg})}}{\cg^{(j)}}$ .

\emph{Inductive case} ($r>2$):
Let $\sigma'$ be the set of outgoing edge from $\root$ in $\greedy{G'}{}$.
Then clearly $\abs{\sigma'}=r\!-\!1>1$.
If $G'$ does not contain a \critical, then $\mst{G'}{}=\greedy{G'}{}$ and we satisfy the induction hypothesis.
Otherwise, $G'$ has a \critical $\cg$.
Then by \cref{lemma:root-reduce}, $\bestDep{G}=\stitch{\bestDep{(\contract{G}{\cg})}}{\cg^{(j)}}$.
\end{proof}

\begin{customthm}{2}
For any graph $G$ with $\mst{G}{}\!=\!\greedy{G}{}$, let $\sigma$ be the set of outgoing edges from $\root$ in $\mst{G}{}$.
If $\abs{\sigma}\!=\!1$, then $\bestDep{G}\!=\!\mst{G}{}$,
otherwise if $G'\!=\!\multidelete{G}{e'}$ for $e'\in\sigma$ that maximizes $\treecost{\greedy{G'}{}}$,
then either $\bestDep{G}\!=\!\bestDep{G'}$ or there exists a \critical $\cg$ in $G'$ such that $\bestDep{G}\!=\!\stitch{\bestDep{(\contract{G}{\cg})}}{\cg^{(j)}}$ where $j$ is the \entersite of $\bestDep{(\contract{G}{\cg})}$.
\end{customthm}
\begin{proof} There are two cases to consider.

\emph{Case 1} ($\abs{\sigma}=1$): Then $\mst{G}{}$ has one edge emanating from the root so clearly $\bestDep{G}=\mst{G}{}$.

\emph{Case 2} ($\abs{\sigma}> 1$). This is immediate from \cref{lemma:root}.
\end{proof}

\newpage

\section{Decoding UD Treebanks}\label{app:table}

\begin{table}[h!]
    \centering
    \resizebox{!}{0.46\textheight}{
    \begin{tabular}{lccccc}
          \bf Language & \bf $\abs{\text{Train}}$ & \bf $\abs{\text{Test}}$ & \bf Malformed Rate & \bf Rel. $\Delta$ UAS & \bf Rel. $\Delta$ Exact Match \\ \midrule
Czech & 68495 & 10148 & 0.45\% & $\ $0.000\% & 0.052\% \\
Russian & 48814 & 6491 & 0.49\% & $\ $0.000\% & 0.027\% \\
Estonian & 24633 & 3214 & 0.93\% & $\ $0.000\% & 0.448\% \\
Korean & 23010 & 2287 & 0.96\% & $\ $0.008\% & 0.366\% \\
Latin & 16809 & 2101 & 0.52\% & $\ $0.018\% & 0.151\% \\
Norwegian & 15696 & 1939 & 0.52\% & -0.014\% & 0.000\% \\
Ancient Greek & 15014 & 1047 & 0.57\% & $\ $0.026\% & 0.186\% \\
French & 14450 & 416 & 1.68\% & -0.021\% & 0.546\% \\
Spanish & 14305 & 1721 & 0.17\% & $\ $0.002\% & 0.000\% \\
Old French & 13909 & 1927 & 0.52\% & $\ $0.031\% & 0.145\% \\
German & 13814 & 977 & 1.54\% & $\ $0.040\% & 0.495\% \\
Polish & 13774 & 1727 & 0.00\% & $\ $0.000\% & 0.000\% \\
Hindi & 13304 & 1684 & 0.18\% & -0.009\% & 0.000\% \\
Catalan & 13123 & 1846 & 0.54\% & $\ $0.002\% & 0.000\% \\
Italian & 13121 & 482 & 0.21\% & -0.010\% & 0.000\% \\
English & 12543 & 2077 & 0.48\% & $\ $0.004\% & 0.217\% \\
Dutch & 12264 & 596 & 0.67\% & $\ $0.039\% & 0.000\% \\
Finnish & 12217 & 1555 & 0.39\% & -0.010\% & 0.000\% \\
Classical Chinese & 11004 & 2073 & 0.96\% & -0.010\% & 0.304\% \\
Latvian & 10156 & 1823 & 0.88\% & -0.012\% & 0.000\% \\
Bulgarian & 8907 & 1116 & 0.27\% & $\ $0.000\% & 0.000\% \\
Slovak & 8483 & 1061 & 0.38\% & $\ $0.008\% & 0.000\% \\
Portuguese & 8328 & 477 & 0.42\% & $\ $0.000\% & 0.000\% \\
Romanian & 8043 & 729 & 0.41\% & $\ $0.000\% & 0.000\% \\
Japanese & 7125 & 550 & 0.00\% & $\ $0.000\% & 0.000\% \\
Croatian & 6914 & 1136 & 0.88\% & $\ $0.027\% & 0.000\% \\
Slovenian & 6478 & 788 & 0.38\% & -0.022\% & 0.000\% \\
Arabic & 6075 & 680 & 0.29\% & $\ $0.004\% & 0.000\% \\
Ukrainian & 5496 & 892 & 0.90\% & $\ $0.032\% & 0.000\% \\
Basque & 5396 & 1799 & 0.67\% & $\ $0.018\% & 0.000\% \\
Hebrew & 5241 & 491 & 1.02\% & $\ $0.009\% & 0.556\% \\
Persian & 4798 & 600 & 0.67\% & -0.007\% & 0.000\% \\
Indonesian & 4477 & 557 & 1.26\% & -0.029\% & 0.000\% \\
Danish & 4383 & 565 & 0.53\% & -0.011\% & 0.000\% \\
Swedish & 4303 & 1219 & 1.23\% & $\ $0.021\% & 0.988\% \\
Old Church Slavonic & 4124 & 1141 & 1.05\% & $\ $0.000\% & 0.128\% \\
Urdu & 4043 & 535 & 1.12\% & -0.029\% & 0.000\% \\
Chinese & 3997 & 500 & 1.80\% & -0.020\% & 0.000\% \\
Turkish & 3664 & 983 & 2.54\% & $\ $0.080\% & 0.513\% \\
Gothic & 3387 & 1029 & 0.78\% & $\ $0.011\% & 0.000\% \\
Serbian & 3328 & 520 & 0.19\% & $\ $0.009\% & 0.446\% \\
Galician & 2272 & 861 & 1.16\% & $\ $0.011\% & 1.282\% \\
North Sami & 2257 & 865 & 1.27\% & $\ $0.000\% & 0.230\% \\
Armenian & 1975 & 278 & 0.00\% & $\ $0.000\% & 0.000\% \\
Greek & 1662 & 456 & 0.44\% & $\ $0.020\% & 0.565\% \\
Uyghur & 1656 & 900 & 0.56\% & $\ $0.024\% & 0.309\% \\
Vietnamese & 1400 & 800 & 3.38\% & -0.076\% & 0.000\% \\
Afrikaans & 1315 & 425 & 6.35\% & $\ $0.011\% & 1.460\% \\
Wolof & 1188 & 470 & 1.49\% & -0.021\% & 0.625\% \\
Maltese & 1123 & 518 & 0.58\% & -0.010\% & 0.000\% \\
Telugu & 1051 & 146 & 0.00\% & $\ $0.000\% & 0.000\% \\
Scottish Gaelic & 1015 & 536 & 0.75\% & -0.024\% & 0.000\% \\
Hungarian & 910 & 449 & 4.23\% & $\ $0.022\% & 0.000\% \\
Irish & 858 & 454 & 2.42\% & $\ $0.000\% & 0.000\% \\
Tamil & 400 & 120 & 0.00\% & $\ $0.000\% & 0.000\% \\
Marathi & 373 & 47 & 2.13\% & $\ $0.000\% & 0.000\% \\
Belarusian & 319 & 253 & 0.79\% & $\ $0.024\% & 0.000\% \\
Lithuanian & 153 & 55 & 7.27\% & -0.317\% & 0.000\% \\
Kazakh & 31 & 1047 & 2.58\% & -0.016\% & 3.226\% \\
Upper Sorbian & 23 & 623 & 6.42\% & $\ $0.178\% & 2.439\% \\
Kurmanji & 20 & 734 & 23.57\% & $\ $0.405\% & 22.222\% \\
Buryat & 19 & 908 & 6.61\% & $\ $0.107\% & 4.082\% \\
Livvi & 19 & 106 & 12.26\% & $\ $0.000\% & 0.000\% \\
\end{tabular}
}
    \caption{Accompanying table for \cref{sec:experiment}}
    \label{tab:full_results}
\end{table}

\end{document}